\documentclass{article}

\pdfoutput=1

\usepackage{times}
\usepackage{graphicx} 
\usepackage{subfigure} 
\usepackage{natbib}
\usepackage{algorithm}
\usepackage{algorithmic}
\usepackage{amsmath,amssymb,amsthm}
\usepackage[utf8]{inputenc}
\usepackage{framed}
\usepackage{enumerate}
\usepackage{enumitem}
\usepackage{multicol}
\usepackage{caption}
\usepackage[dvipsnames]{xcolor}

\usepackage[accepted]{icml2016}

\newcommand\theTitle{Estimation from Indirect Supervision with Linear Moments}

\icmltitlerunning{\theTitle}

\newcommand\sF{\ensuremath{\mathcal{F}}}

\newcommand\sI{\ensuremath{\mathcal{I}}}

\newcommand\sN{\ensuremath{\mathcal{N}}}
\newcommand\sO{\ensuremath{\mathcal{O}}}
\newcommand\sP{\ensuremath{\mathcal{P}}}

\newcommand\sX{\ensuremath{\mathcal{X}}}
\newcommand\sY{\ensuremath{\mathcal{Y}}}

\newcommand\BP{\ensuremath{\mathbb{P}}}

\newcommand\Fig[4]{\begin{figure}[ht] \begin{center} \includegraphics[scale=#2]{#1} \end{center} \vspace{-1.0em} \caption{\label{fig:#3} #4} \vspace{-.5em} \end{figure}}
\newcommand\FigTop[4]{\begin{figure}[t] \begin{center} \includegraphics[scale=#2]{#1} \end{center} \vspace{-1.0em} \caption{\label{fig:#3} #4} \vspace{-.5em} \end{figure}}

\DeclareMathOperator*{\tr}{tr}

\DeclareMathOperator*{\cov}{Cov} 
\DeclareMathOperator*{\diag}{diag} 
\newcommand\p[1]{\ensuremath{\left( #1 \right)}} 
 
\newcommand\pb[1]{\ensuremath{\left[ #1 \right]}}

\newcommand\inv[1]{\ensuremath{\frac{1}{#1}}}
\newcommand\half{\ensuremath{\frac{1}{2}}}
\newcommand\R{\ensuremath{\mathbb{R}}}

\newcommand\eqdef{\ensuremath{\stackrel{\rm def}{=}}} 
\newcommand{\1}{\mathbb{I}}

\newcommand\refeqn[1]{(\ref{eqn:#1})}
\newcommand\refeqns[2]{(\ref{eqn:#1}) and (\ref{eqn:#2})}

\newcommand\refsec[1]{Section~\ref{sec:#1}}

\newcommand\reffig[1]{Figure~\ref{fig:#1}}

\newcommand\reftab[1]{Table~\ref{tab:#1}}

\newcommand\refprop[1]{Proposition~\ref{prop:#1}}

\ifthenelse{\isundefined{\definition}}{\newtheorem{definition}{Definition}}{}
\ifthenelse{\isundefined{\assumption}}{\newtheorem{assumption}{Assumption}}{}
\ifthenelse{\isundefined{\hypothesis}}{\newtheorem{hypothesis}{Hypothesis}}{}
\ifthenelse{\isundefined{\proposition}}{\newtheorem{proposition}{Proposition}}{}
\ifthenelse{\isundefined{\theorem}}{\newtheorem{theorem}{Theorem}}{}
\ifthenelse{\isundefined{\lemma}}{\newtheorem{lemma}{Lemma}}{}
\ifthenelse{\isundefined{\corollary}}{\newtheorem{corollary}{Corollary}}{}
\ifthenelse{\isundefined{\alg}}{\newtheorem{alg}{Algorithm}}{}
\ifthenelse{\isundefined{\example}}{\newtheorem{example}{Example}}{}

\newcommand\cvd{\ensuremath{\xrightarrow{d}}} 
\newcommand\cvP{\ensuremath{\xrightarrow{P}}}

\newcommand{\E}{\ensuremath{\mathbb{E}}} 
\newcommand\KL[2]{\ensuremath{\text{KL}\left( #1 \| #2 \right)}}

\newcommand\posstart{j_\text{start}}
\newcommand\posend{j_\text{end}}

\newcommand\eg{\textit{e.g.}\ }

\newcommand\itr[1]{{({#1})}}

\newcommand\mX{\mathcal{X}}
\newcommand\mY{\mathcal{Y}}

\newcommand\mO{\mathcal{O}}
\newcommand\mG{\mathcal{G}}

\newcommand \MargML {\hat\theta_{\text{\rm marg}}}
\newcommand \MomML {\hat\theta_{\text{\rm mom}}}

\newcommand \vMargML {\Sigma_{\text{\rm marg}}}
\newcommand \vMomML {\Sigma_{\text{\rm mom}}}

\newcommand\pervalue{\text{\rm pv}}
\newcommand\coordrel{\text{\rm cr}}

\newcommand\eff{\text{\rm Eff}}

\newcommand{\vecsquare}{^{\otimes 2}}

\newcommand{\diffp}{\alpha}
\newcommand{\diameter}{\delta}

\newcommand\lone[1]{\|{#1}\|_1}
\newcommand{\yell}[1]{\textbf{#1} \;\; }

\newcommand{\argmax}{\mathop{\rm arg\hspace{.1em}max}}
\newcommand{\argmin}{\mathop{\rm arg\hspace{.1em}min}}

\newcommand\SpaceFig[4]{\begin{figure}[ht] \vspace{-1.0em} \begin{center} \includegraphics[scale=#2]{#1} \end{center} \vspace{-2.0em} \caption{\label{fig:#3} #4} \vspace{-1.0em} \end{figure}}

\newcommand\SpaceFigTop[4]{\begin{figure}[t] \begin{center} \includegraphics[scale=#2]{#1} \end{center} \vspace{-1.0em} \caption{\label{fig:#3} #4} \vspace{-1.5em} \end{figure}}

\begin{document} 

\twocolumn[
\icmltitle{\theTitle}

\icmlauthor{Aditi Raghunathan}{aditir@stanford.edu}
\icmlauthor{Roy Frostig}{rf@cs.stanford.edu}
\icmlauthor{John Duchi}{jduchi@stanford.edu}
\icmlauthor{Percy Liang}{pliang@cs.stanford.edu}
\icmladdress{Stanford University, Stanford, CA}

\icmlkeywords{estimation, indirect supervision, method of moments,
  graphical models, latent variable models, structured prediction,
  local privacy}

\vskip 0.3in
]

\begin{abstract} 
In structured prediction problems where we have indirect supervision of the
output, maximum marginal likelihood faces two computational obstacles:
non-convexity of the objective and intractability of even a single gradient
computation.  In this paper, we bypass both obstacles for a class of what we
call linear indirectly-supervised problems.  Our approach is simple: we
solve a linear system to estimate sufficient statistics of the model,
which we then use to estimate parameters via convex optimization.  We
analyze the statistical properties of our approach and show empirically that
it is effective in two settings: learning with local privacy constraints
and learning from low-cost count-based annotations.\footnote{
  This is an extended and updated version of our paper
  in \textit{Proceedings of the $\mathit{33}^{rd}$
    International Conference on Machine Learning}, New York, NY, USA,
  2016. JMLR: W\&CP volume 48.  Copyright 2016 by the author(s).}

\end{abstract} 

\section{Introduction}
\label{sec:intro}

Consider the problem of estimating a probabilistic graphical model from
\textit{indirect supervision}, where only a partial view of the variables is
available.  We are interested in indirect supervision for two reasons:
first, one might not trust a data collector and wish to use privacy
mechanisms to reveal only partial information about sensitive data
\citep{warner1965randomized, evfimievski2004privacy, dwork2006calibrating,
  duchi2013local}.  Second, if data is generated by human annotators, say in
a crowdsourcing setting, it can often be more cost-effective to solicit
lightweight annotations
\citep{maron1998framework,mann08ge,quadrianto08labels,liang09measurements}.
In both cases, we trade statistical efficiency for another
resource: privacy or annotator cost.

Indirect supervision is naturally handled by defining a latent-variable model
where the structure of interest is treated as a latent variable.
While statistically sensible,
learning latent-variable models poses two computational challenges.
First, maximum marginal likelihood requires non-convex optimization,
where popular procedures such as gradient descent or Expectation Maximization (EM)
\citep{demp1977em} are only guaranteed to converge to local optima.
Second, even the computation of the gradient or performing the E-step can be intractable,
requiring probabilistic inference on a loopy graphical model induced by the indirect supervision
\citep{chang07constraint,graca08em,liang09measurements}.

In this paper,
we propose an approach that bypasses both computational obstacles
for a class which we call \emph{linear indirectly-supervised} learning problems.
We lean on the method of moments \citep{pearson1894contributions},
which has recently led to advances in learning latent-variable models
\citep{hsu09spectral,anandkumar13tensor,chaganty2014graphical},
although we do not appeal to tensor factorization.
Instead, we express indirect supervision as a linear combination of the
sufficient statistics of the model,
which we recover by solving a simple noisy linear system.
Once we have the sufficient statistics, we use convex optimization to
solve for the model parameters.
The key is that while supervision per example is indirect and leads to intractability,
\emph{aggregation} over a large number of examples renders the problem tractable.

While our moments-based estimator yields computational benefits,
we suffer some statistical loss relative to maximum marginal likelihood.
In \refsec{variance}, we compare the asymptotic variance of marginal-likelihood and moment-based
estimators, and provide some geometric insight
into their differences in \refsec{two_step}.
Finally, in \refsec{experiments}, we apply our framework empirically to
our two motivating settings:
(i) learning a regression model under local privacy constraints,
and (ii) learning a part-of-speech tagger with lightweight annotations.
In both applications, we show that our moments-based estimator obtains good accuracies.

\section{Setup}
\label{sec:setup}

\Fig{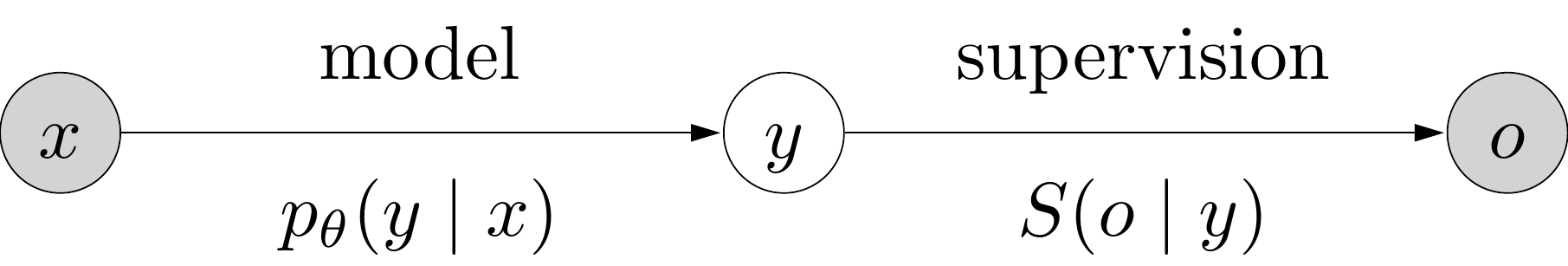}{0.4}{graphicalModel}{
We solve a structured prediction problem from $x$ to $y$.
During training, we observe not $y$, but indirect supervision $o$.}

\paragraph{Notation.}

We use superscripts to enumerate instances in a data sample (\eg $x^\itr{1},
\dots, x^\itr{n}$), and square-bracket indexing to enumerate components of a
vector or sequence: $x[b]$ denotes the component(s) of $x$ associated with
$b$.  For a real vector $x \in \R^d$, we let $x\vecsquare
\eqdef x x^\top$.
 
\paragraph{Model.}
Consider the structured prediction task of mapping an input $x \in
\mathcal{X}$ to some output $y \in \mathcal{Y}$.  We model this mapping
using a conditional exponential family
\begin{align}
  \label{eqn:exp-fam}
  p_\theta(y \mid x) = \exp\{\phi(x,y)^\top \theta - A(\theta; x)\},
\end{align}
where $\phi : \sX \times \sY \to \R^d$ is the feature mapping, $\theta \in
\R^d$ is the parameter vector, and $A(\theta; x) = \log \sum_{y \in \mY}
\exp\{\phi(x,y)^\top \theta\}$ is the log-partition function.  For
concreteness, we specialize to conditional random fields (CRFs)
\citep{lafferty01crf} over collections of $K$-variate labels, where $x =
(x[1], \dots, x[L])$ and $y = (y[1], \dots, y[L]) \in [K]^L$; here $L$ is
the number of variables and $K$ is the number of possible labels per
variable. We let $C \subseteq 2^{[L]}$ be the set of cliques in the CRF, so
that the features decompose into a sum over cliques: $\phi(x, y) = \sum_{c
  \in C} \phi_c(y[c], x, c)$.  As one particular example, if $C$ consists of
all nodes $\{ \{1\}, \dots, \{L\} \}$ and edges between adjacent nodes $\{
\{1, 2\}, \dots, \{L-1, L\} \}$, the CRF is chain-structured.

\paragraph{Learning from indirect supervision.}
In the \emph{indirectly supervised} setting that is the focus of this paper,
we do not have access to $y$ but rather only observations $o \in
\mathcal{O}$, where $o$ is drawn from a known \emph{supervision distribution}
$S(o \mid y)$.

For each $i = 1, \dots, n$, let $(x^{(i)}, y^{(i)})$ be drawn from some
unknown data-generating distribution $p^*$ (by default, we do not assume the
CRF is well-specified), and $o^{(i)}$ is drawn according to $S(\cdot \mid
y^{(i)})$ as in \reffig{graphicalModel}. The learning problem is then the
natural one: given the training examples $(x^{(1)}, o^{(1)}), \dots,
(x^{(n)}, o^{(n)})$, we wish to produce an estimate $\hat\theta$ for
the model~\eqref{eqn:exp-fam}.

\paragraph{Maximum marginal likelihood.}

The classic paradigm is to maximize the marginal likelihood:
\begin{align}
  \label{eqn:margML}
  \MargML \eqdef \argmax_\theta \hat\E\pb{\log \sum_{y \in \sY} S(o \mid y) p_\theta(y \mid x)},
\end{align}
where $\hat\E$ denotes an expectation over the training sample.  While
$\MargML$ is often statistically efficient, there are two computational
difficulties associated with this approach:

\begin{enumerate} 
\setlength\itemsep{0em}
\item The log-likelihood objective \refeqn{margML} is typically non-convex,
so computing $\MargML$ exactly is in general intractable; see Section~\ref{sec:two_step} for a
more detailed discussion.
Local algorithms like Expectation Maximization \citep{demp1977em}
are only guaranteed to converge to local optima.

\item Computing the gradient or the E-step requires computing
$p(y \mid x, o) \propto p_\theta(y \mid x) S(o \mid y)$,
which is intractable, not due to the model $p_\theta$, but to the
supervision $S$.
This motivates a number of relaxations
\citep{graca08em,liang09measurements,steinhardt2015relaxed},
but there are no guarantees on approximation quality.
\end{enumerate}

\paragraph{Our approach: moment-based estimation.}

\FigTop{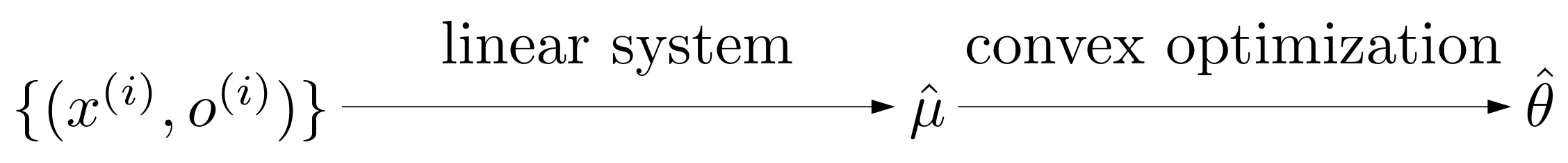}{0.35}{twoStep}{ Our approach is to (i) solve
  a linear system based on the data $\{ (x^{(i)}, o^{(i)}) \}$ to estimate
  the sufficient statistics $\hat\mu$, then (ii) use convex
  optimization to estimate the model parameters $\hat\theta$.  }

We present a simple approach to circumvent the above issues for a
restricted class of models, in the same vein as \citet{chaganty2014graphical}.
To begin, consider the fully-supervised
setting, where we observe examples $\{(x^{(i)}, y^{(i)})\}$.  In this case,
we could maximize the likelihood $\hat\E[\log p_\theta(y \mid x)]$,
solving $\argmax_\theta \{\hat\mu^\top \theta - \hat\E[A(\theta; x)]\}$,
where $\hat\mu \eqdef \hat\E[\phi(x,y)] \in \R^d$ are the sufficient
statistics, which converge to $\mu^* \eqdef \E[\phi(x,y)]$.  Therefore, if
we could construct a consistent estimate $\hat\mu$ of $\mu^*$, then we could
solve the same convex optimization problem used in the fully-supervised
estimator.

Of course, we do not have access to $\hat\E[\phi(x,y)]$.  Instead, in our
(linearly) indirectly supervised setting, we are able to define an
\emph{observation function} $\beta(x, o) \in \R^d$ which is nonetheless in
expectation equal to the population sufficient statistics:
\begin{align}
  \label{eqn:betaGivesMu}
  \E[\beta(x, o) \mid x] = \phi(x, y), \quad
  \E[\beta(x, o)] = \mu^*.
\end{align}
In general, we construct $\beta(x, o)$ by solving a linear system.
Putting the pieces together yields our estimator (\reffig{twoStep}):
\begin{enumerate} 
\setlength\itemsep{0pt}
\item Sufficient statistics: $\hat\mu = \hat\E[\beta(x, o)]$.
\item Parameters: $\MomML \eqdef \arg\max_\theta \{\hat\mu - \hat\E[A(\theta; x)]\}$.
\end{enumerate}
In the next two sections, we describe the observation function $\beta(x,o)$ for
learning with local privacy (\refsec{localPrivacy}) and lightweight annotations (\refsec{annotations}).

\section{Learning under local privacy}
\label{sec:localPrivacy}

Suppose we wish to estimate a conditional distribution $p_\theta(y \mid x)$,
where $x$ is non-sensitive information about an individual and $y$ contains
sensitive information (for example, income or disease status). Individuals,
because of a variety of reasons---mistrust, embarrassment, fear of
discrimination---may wish to keep $y$ private and instead release some $o \sim S(\cdot \mid y)$.
To quantify the amount of privacy afforded by $S$, we turn to the
literature on privacy in databases and theoretical computer
science~\citep{evfimievski2004privacy,dwork2006calibrating} and say that $S$ is
$\diffp$-differentially private if any two $y,y'$ have comparable probability (up to a factor of $\exp(\alpha)$)
of generating $o$:
\begin{align}
  \sup_{o,y,y'} \frac{S(o \mid y)}{S(o \mid y')} \le \exp(\diffp).
\end{align}
What $S$ should we employ?
We first explore the classical randomized response (RR) mechanism (\refsec{classicRR}),
and then develop a new mechanism that leverages the graphical model structure (\refsec{structuredRR}).

\subsection{Classic randomized response}
\label{sec:classicRR}

\citet{warner1965randomized} proposed the
now-classical randomized response technique,
which proceeds as follows:~For some fixed (generally
small) $\epsilon > 0$, the respondent reveals $y$ with probability
$\epsilon$ and with probability $1 - \epsilon$ draws a sample from a (known)
base distribution $u$---generally uniform---over $\mY$. Formally,
the classical randomized response supervision is
\begin{align}
  \label{eqn:rr}
  S(o \mid y) \eqdef \epsilon \1[o = y] + (1 - \epsilon) u(o).
\end{align}

\paragraph{Estimation.}

Our goal is to construct a function $\beta$ satisfying \refeqn{betaGivesMu}.
Towards that end, let us start with what we can estimate and expand based on \refeqn{rr}:
\begin{align}
  \label{eqn:rrExpected}
  \E[\phi(x, o)] = \epsilon \E[\phi(x,y)] + (1 - \epsilon) \E[\phi(x, y')],
\end{align}
where $y' \sim u$.  Rearranging \refeqn{rrExpected}, we see that we can solve
for $\mu^* = \E[\phi(x, y)]$.
Indeed, if we define the observation function:
\begin{align}
  \label{eqn:beta-privacy}
  \beta(x, o) \eqdef
  \frac{\phi(x, o) - (1 - \epsilon) \E[\phi(x, y') \mid x]}{\epsilon},
\end{align}
we can verify that $\E[\beta(x, o)] = \mu^*$.

\paragraph{Privacy.}

We can check that the ratio $S(o \mid y) / S(o \mid y') \le 1 +
\frac{\epsilon + (1 - \epsilon) u(o)}{(1 - \epsilon) u(o)}$, so classical
randomized response is $\frac{\epsilon}{(1 - \epsilon) \min_o
  u(o)}$-differentially private.  For any distribution $u$, this value is at
least $\frac{\epsilon}{1 - \epsilon} |\sY|$, a linear dependence on $|\sY|$.
In classical randomized response settings, $|\sY| = 2$, which is unproblematic.
In contrast, in structured
prediction settings, the number of labels is exponential in the
number of variables ($|\sY| = K^L$),
so we must take $\epsilon = O\p{\frac{\diffp}{|\mY|}}$.
The asymptotic variance of $\MomML$
scales as $\epsilon^{-2}$ (as will be shown in \refsec{variance}),
which makes classical randomized response unusable for structured prediction.

\subsection{Structured randomized response}
\label{sec:structuredRR}

With this difficulty in mind, we recognize that we must somehow leverage the
structure of the sufficient statistics vector $\phi(x, y)$ to perform
estimation.
In particular, we show that the supervision should only depend on the
sufficient statistics:

\begin{proposition}
  \label{prop:sufficient_channel}
  Let $\mO$ be the set in which observations live.
  For any privacy mechanism $S(o \mid x, y)$ that is $\diffp$-differentially
  private, there exists a mechanism $S'(o' \mid \phi(x,y))$
  that is at least $\diffp$-differentially private, and for
  any set $A \subseteq \mO$, we have
  \begin{align}
    \label{eqn:preserveXO}
    \BP[o' \in A \mid x] = \BP[o \in A \mid x],
  \end{align}
  where $o' \sim S'(\cdot \mid \phi(x, y))$ and $o \sim S(\cdot \mid x, y)$.
\end{proposition}
In short, we can always construct $S'$ that only uses the
sufficient statistics $\phi(x, y)$ but yields the same joint distribution over
the pairs $(x, o)$.  Furthermore, $S'$ is \emph{at least} as private as the original
mechanism $S$. See Appendix~\ref{sec:privacy_proofs} for a proof.

This motivates a focus exclusively on mechanisms that use sufficient
statistics, and in particular, we consider the following two
\emph{structured randomized response} mechanisms.  Our schemes are
both two-phase procedures that first binarize the sufficient
statistics, and then release a set of observations inspired by
\citeauthor{duchi2013local}'s minimax optimal approach to estimating a
multinomial distribution.
For $t > 0$, let $q_t \eqdef \frac {e^{t/2}} {1 + e^{t/2}}$. Assume
each coordinate of the statistics $\phi$ lies in the interval $[0, c]$
for some positive scalar $c$. For $i \in \{1,\dots,d\}$, draw $\tilde
o[i]$ as a Bernoulli variable with bias $\tfrac 1 c
\phi(x,y)[i]$. Then:

\yell{(Coordinate release)}
  Draw a coordinate $j \in \{1,\dots,d\}$ from a distribution
  $p_\coordrel$. Set $o_\coordrel = \tilde o[j]$ with probability
  $q_\diffp$, otherwise $o_\coordrel = 1 - \tilde o[j]$. Release the
  pair $(j, o_\coordrel)$.

\yell{(Per-value $\phi$-RR)}
  Denote by $\Omega(x,y)$ the support of $\tilde o$ given $x,y$, let
  $\diameter \eqdef \sup_{x,y,\tilde o \in \Omega(x,y)} \lone{\tilde
    o}$, and take any $\bar \diameter \ge \diameter$. For $j = 1,
  \dots, d$, set $o_\pervalue[j] = \tilde o[j]$ with probability
  $q_{\diffp/\bar\diameter}$, otherwise $o_\pervalue[j] = 1 - \tilde
  o[j]$. Release the vector $o_\pervalue$.

Both are $\diffp$-differentially private (see Appendix~\ref{sec:privacy_proofs}).
For coordinate release, define the observation function
\begin{align*}
  \beta_{\coordrel}(x, (j, o_\coordrel)) \eqdef
  p_\coordrel^{-1}(j) \;
  \frac { o_\coordrel - 1 + q_{\diffp} } { 2 q_{\diffp} - 1 } \; c \;
  e_j,
\end{align*}
where $e_j$ denotes the $j$'th standard basis vector.
For the per-value statistics scheme, define the observation function,
\begin{align}
  \label{eqn:pervalue-beta-construction}
  \beta_{\pervalue}(x, o_\pervalue) \eqdef
  \frac { o_\pervalue - \mathbf{1} + q_{\diffp/\bar\diameter} \mathbf{1} } { 2 q_{\diffp/\bar\diameter} - 1 } \; c.
\end{align}
In either case, we have that $\E[\beta(x, o)] = \mu^*$, as required by
\refeqn{betaGivesMu} for $\MomML$ to be consistent.

The two schemes offer a tradeoff: when $\tilde{o}$ is dense,
coordinate release is advantageous, as our best norm bound
$\bar\diameter$ may be as large as the dimension $d$, so although we
reveal only a single coordinate at a time, we noise it by a
lower-variance distribution $q_{\diffp}$ rather than the
$q_{\diffp/d}$ noise of the per-value scheme. Meanwhile, per-value
$\phi$-RR enjoys lower variance when $\tilde{o}$ has low $\ell_1$
norm.
The latter case arises, for instance, if $\phi$ is a sparse binary
vector as is common in structured prediction.
Appendix~\ref{sec:variance_appendix} and~\ref{sec:comparison} present more
details about this tradeoff offered by the schemes.

Summarizing, we have three randomized response schemes. Classical RR
appeals only in unstructured problems with few outputs $\mY$. In the
structured setting, we can move to the sufficient statistics $\phi$ by
Proposition~\ref{prop:sufficient_channel}, and exploit their structure
with either of two schemes based on our knowledge of the 1-norm or
sparsity of statistics $\phi$.

\section{Learning with lightweight annotations}
\label{sec:annotations}
\vspace{-2em}

\begingroup
\setlength{\tabcolsep}{2pt}
\begin{center}
\begin{table}
\hspace{.35cm}
{ \fontsize{0.3cm}{0.4cm} \selectfont
\begin{tabular}{c | c c c c c c c c c }
$x$    & The & quick & brown & [\textbf{fox} & \textbf{jumps} & \textbf{over} & \textbf{the} & \textbf{lazy} & \textbf{dog}] \\ \hline
$y$    & DT & JJ & JJ & [\textbf{NN} & \textbf{VBZ} & \textbf{IN} & \textbf{DT} & \textbf{JJ} & \textbf{NN}] \\ \hline
$o$   & & & &  &  & \multicolumn{2}{c}{\# \textbf{\textcolor{blue}{NN}} = 2} & & \\
\end{tabular}
\caption{Part-of-speech tagging with region annotations.
  An annotator is given a region (bold, in brackets)
  and asked to count the number of times particular tags (e.g., NN) occurs.
\label{tab:annotations}}
}
\end{table}
\end{center}
\endgroup

For a sequence labeling task, e.g., part-of-speech (POS) tagging,
it can be tedious to obtain fully-labeled input-output sequences for training.
This motivates a line of work which attempts to learn
from weaker forms of annotation \citep{mann08ge,haghighi06prototype,liang09measurements}.
We focus on \emph{region annotations}, where an annotator
is asked to examine only a particular subsequence of the input
and count the number of occurrences of some label (e.g., nouns).
The rationale is that it is cognitively easier for the annotator to focus on one label at a time
rather than annotating from a large tag set, and
physically easier to hit a single yes/no or counter button
than to select precise locations,
especially in mobile-based crowdsourcing interfaces \citep{vaish2014twitch}.
See \reftab{annotations} for an example.

More formally, the supervision $S(o \mid y)$ is defined as follows:
First, choose the starting position $\posstart$ uniformly from $\{ 1, \dots, L - w \}$,
and set the ending position $\posend = \posstart + w - 1$,
where $w$ is a fixed window size.
Let $r = \{ \posstart, \dots, \posend \}$ denote this region.
Next, choose a subset of tags $B$ uniformly from the tag set (e.g., $\{ \text{NN}, \text{DT} \}$).
From here, the observation $o$ is generated deterministically:
For each tag $b \in B$, the annotator counts the number of occurrences in the region:
$N[b] = |\{ j \in r : y[j] = b \}|$.
The final observation is $o = (r, B, N)$.

In this setting, not only is the marginal likelihood non-convex,
inference requires summing over possible ways of realizing the counts,
which is exponential either in the window size $w$ and $|B|$.

\paragraph{Estimation.}

For our estimator to work, we make two assumptions:
\begin{enumerate}[noitemsep,nolistsep]
  \item The node potentials only depend on $x[j]$:
$\phi_j(y[j], x, j) = f(x[j], y[j])$; and
  \item Under the true conditional distribution, $y[j]$ only depends on $x[j]$:
$p^*(y[j] \mid x) = p^*(y[j] \mid x[j])$.
\end{enumerate}
These are admittedly strong independence assumptions similar to IBM model 1 for word alignment \citep{brown93mt}
or the unordered translation model of \citet{steinhardt2015relaxed}.
Even though our model is fully factorized and lacks
edge potentials, inference $p_\theta(y \mid x,o)$ is expensive as conditioning on the indirect supervision $o$ couples all of the $y$ variables.
This typically calls for approximate inference techniques
common to the realm of structured prediction.
\citet{steinhardt2015relaxed} developed a relaxation to cope with this supervision,
but this still requires approximate inference via sampling and non-convex optimization.

In contrast to the local privacy examples, the new challenge is that the
observation $o$ does not provide enough information to evaluate a single node
potential, even stochastically.
So we cannot directly write $\mu^*$ in terms
of functions of the observations.
As a bridge, define the localized conditional distributions:
$w^*(a, b) \eqdef \BP[y[j] = b \mid x[j] = a]$,
which by assumption 2 specify the entire conditional distribution.
The sufficient statistics $\mu^*$ can be written as in terms of $w^*$:
\begin{align}
  \label{eqn:wToMu}
  \mu^* = \E\left[\sum_{j=1}^L \sum_b w^*(x[j], b) f(x[j], b) \right].
\end{align}

We now define constraints that relate the observations $o$ to $w^*$.
Recall that each observation $o$ includes a region $r$, a tag $b$,
and a vector of counts $N = [\sum_{j \in r} \1[y[j] = b]]_{b \in B}$, one for each tag $b$.
For each input $x \in \sX$ and tag $b$, we have the identity:
\begin{align}
  \label{eqn:nToW}
  \E[N[b] \mid x, r] = \sum_{j \in r} w^*(x[j], b).
\end{align}

While we do not observe the LHS, we observe $N[b]$, which is unbiased estimate
of the RHS of \refeqn{nToW}.
We can therefore solve a regression problem with response $N$
to recover a consistent estimate $\hat w$ of $w^*$:
\begin{align}
  \hat w = \arg\min_w \sum_{i=1}^n \sum_{b \in B^{(i)}} \p{\sum_{j \in r^{(i)}} w(x^{(i)}[j], b) - N^{(i)}[b]}^2.
\end{align}
For instance, the example in \reftab{annotations} contributes:
$(\BP[\text{NN} \mid \text{fox}] + \BP[\text{NN} \mid \text{jumps}] + \cdots + \BP[\text{NN} \mid \text{dog}] - 2)^2$.
Finally, we plug in $\hat w$ into \refeqn{wToMu} obtain $\hat \mu$.

\section{Asymptotic analysis}
\label{sec:variance}

We have two estimators:
maximum marginal likelihood ($\MargML$),
which is difficult to compute,
requiring non-convex optimization and possibly intractable inference;
and our moments-based estimator ($\MomML$),
which is easy to compute,
requiring only solving a linear system and convex optimization.
In this section, we study and compare the \emph{statistical efficiency}
of $\MargML$ and $\MomML$.
For simplicity, we focus on unconditional setting where $x$ is empty,
and omit $x$ in this section.
We also assume our exponential family model is well-specified and that
$\theta^*$ are the true parameters.  All expectations are taken with respect
to $y \sim p_{\theta^*}$.

Recall from \refeqn{betaGivesMu} that $\E[\beta(o) \mid x] = \phi(y)$.
We can therefore think of $\beta(o)$ as a ``best guess'' of $\phi(y)$.
The following lemma provides the asymptotic variances of the estimators:
\begin{proposition}[General asymptotic variances] 
  \label{prop:variance}
  Let $\sI \eqdef \cov[\phi(y)]$ be the Fisher information matrix.
  Then
  \begin{align*}
    \sqrt{n}(\MargML - \theta^*) \cvd \sN(0, \vMargML) ~~ \mbox{and}\\
    \sqrt{n}(\MomML - \theta^*) \cvd \sN(0, \vMomML),
  \end{align*}
  where the asymptotic variances are
  \begin{align}
    \label{eqn:margVariance}
    \vMargML &= (\sI - \E[\cov[\phi(y) \mid o]])^{-1}, \\
    \label{eqn:momVariance}
    \vMomML &= \sI^{-1} + \sI^{-1} \E[\cov[\beta(o) \mid y]] \sI^{-1}.
  \end{align}
\end{proposition}

Let us compare the asymptotic variances of $\MargML$ and $\MomML$ to that of
the fully-supervised maximum likelihood estimator $\hat{\theta}_{\rm full}$,
which has access to $\{(x^{(i)}, y^{(i)})\}$, and satisfies
$\sqrt{n}(\hat{\theta}_{\rm full} - \theta^*) \cvd \sN(0, \sI^{-1})$.

Examining the asymptotic variance of $\MargML$ \refeqn{margVariance}, we see
that the loss in statistical efficiency with respect to maximum likelihood is the
amount of variation in $\phi(y)$ not explained by $o$, $\cov[\phi(y) \mid o]$.
Consequently, if $y$ is simply deterministic given
$o$, then $\cov[\phi(y) \mid o] = 0$, and $\MargML$ achieves the
statistically efficient asymptotic variance $\sI^{-1}$.

The story with $\MomML$ is dual to that for the marginal
likelihood estimator. Considering the second term in
expression~\eqref{eqn:momVariance}, we see that the loss of efficiency due
to our observation model grows linearly in the variability of the
observations $\beta(o)$ not explained by $y$. Thus, unlike $\MargML$, even
if $y$ is deterministic given $o$ (so $o$ reveals full information about
$y$), we do not recover the efficient covariance $\sI^{-1}$.
As a trivial example, let $\phi(y) = y \in
\{0,1\}$ and the observation $o = [y, y + \eta]^\top$ for $\eta \sim \sN(0, 1)$,
so that $o$ contains a faithful copy of $y$,
and let $\beta(o) = \frac{o[1] + o[2]}{2} = y + \frac{\eta}{2}$. Then
$\E[\cov[\beta(o) \mid y]] = \frac{1}{4}$, and
the asymptotic relative efficiency of $\MomML$ to $\MargML$
is $\frac{1}{1 + \sI^{-1} / 4}$.
Roughly, $\MargML$ integrates posterior information about $y$ better
than $\MomML$ does.

\begin{proof}
To compute $\vMargML$, we follow standard arguments in
\citet{vaart98asymptotic}.  If $\ell(o; \theta)$ is the marginal
log-likelihood, then a straightforward calculation yields $\nabla \ell(o,
\theta^*) = \E[\phi(y) \mid o] - \E[\phi(y)]$.  The asymptotic variance is
the inverse of $\E[\nabla \ell(o, \theta^*) \nabla \ell(o, \theta^*)^\top] =
\cov[\E[\phi(y) \mid o]]$; applying the variance decomposition $\sI =
\cov[\E[\phi(y) \mid o]] + \E[\cov[\phi(y) \mid o]]$
gives~\eqref{eqn:margVariance}.

To compute $\vMomML$, recall that the moments-based estimator computes
$\hat\mu = \hat\E[\beta(o)]$ and $\hat\theta = (\nabla A)^{-1}(\hat\mu)$.
Apply the delta method, where $\nabla (\nabla A)^{-1}(\mu^*) = (\nabla^2
A(\theta^*))^{-1} = \cov[\phi(y)]^{-1}$.  Finally, decompose $\cov[\beta(o)]
= \cov[\E[\beta(o) \mid y]] + \E[\cov[\beta(o) \mid y]]$ and recognize that
$\E[\beta(o) \mid y] = \phi(y)$ to obtain~\eqref{eqn:momVariance}.
\end{proof}

\paragraph{Randomized response.}

To obtain concrete intuition for \refprop{variance},
we specialize to the case where $S$ is the randomized response \refeqn{rr}.
In this setting, 
$\beta(o) = \epsilon^{-1} \phi(o) - h$ for some constant vector $h$.
Recall the supervision model:
$z \sim \text{Bernoulli}(\epsilon)$,
$o = y$ if $z = 1$ and $o = y' \sim u$ if $z = 0$.

\begin{lemma}[asymptotic variances (randomized response)] 
  \label{lem:variancerr}
  Under the randomized response model of \refeqn{rr}, the asymptotic variance
  of $\MargML$ is
  \begin{align}
    \label{eqn:variances-rr}
    \lefteqn{\vMomML = \sI^{-1} + \sI^{-1} H  \sI^{-1},} \\
    H &\eqdef \frac{1-\epsilon}{\epsilon^2}\cov[\phi(y')]
      + \frac{1 - \epsilon}{\epsilon}
      \E[(\phi(y) - \E[\phi(y')])\vecsquare]. \nonumber
  \end{align}
\end{lemma}
The matrix $H$ governs the loss of efficiency, which stems from two sources:
(i) $\cov[\phi(y')]$, the variance when we sample $y' \sim u$; and
(ii) the variance in choosing between $y$ and $y'$.
If $y'$ and $y$ have the same distribution,
then $H = \sI \frac{1 - \epsilon^2}{\epsilon^2}$
and $\vMomML = \epsilon^{-2} \sI^{-1}$.

\begin{proof}
  We decompose $\cov[\beta(o) \mid y]$ as
  \begin{align*}
    \lefteqn{\E[\cov[\beta(o) \mid y, z] \mid y]
      + \cov[\E[\beta(o) \mid y, z] \mid y]} \\
    & = \E[\epsilon \cdot 0
      + (1 - \epsilon) \cov[\epsilon^{-1} \phi(y')]
      \mid y] \\
    & ~~ + \cov[z\epsilon^{-1} \phi(y)
      + (1 - z) \E[\epsilon^{-1}\phi(y')] \mid y] \\
    & = \frac{1 - \epsilon}{\epsilon^2} \cov[\phi(y')]
    + \frac{1 - \epsilon}{\epsilon} (\phi(y) - \E[\phi(y')])\vecsquare,
  \end{align*}
  where we used
  $\beta(y) = \epsilon^{-1}\phi(y) - h$.
\end{proof}

\begin{figure}
  \begin{center}
    \begin{tabular}{cc}
      \hspace{-.4cm}
      \includegraphics[width=.54\columnwidth]{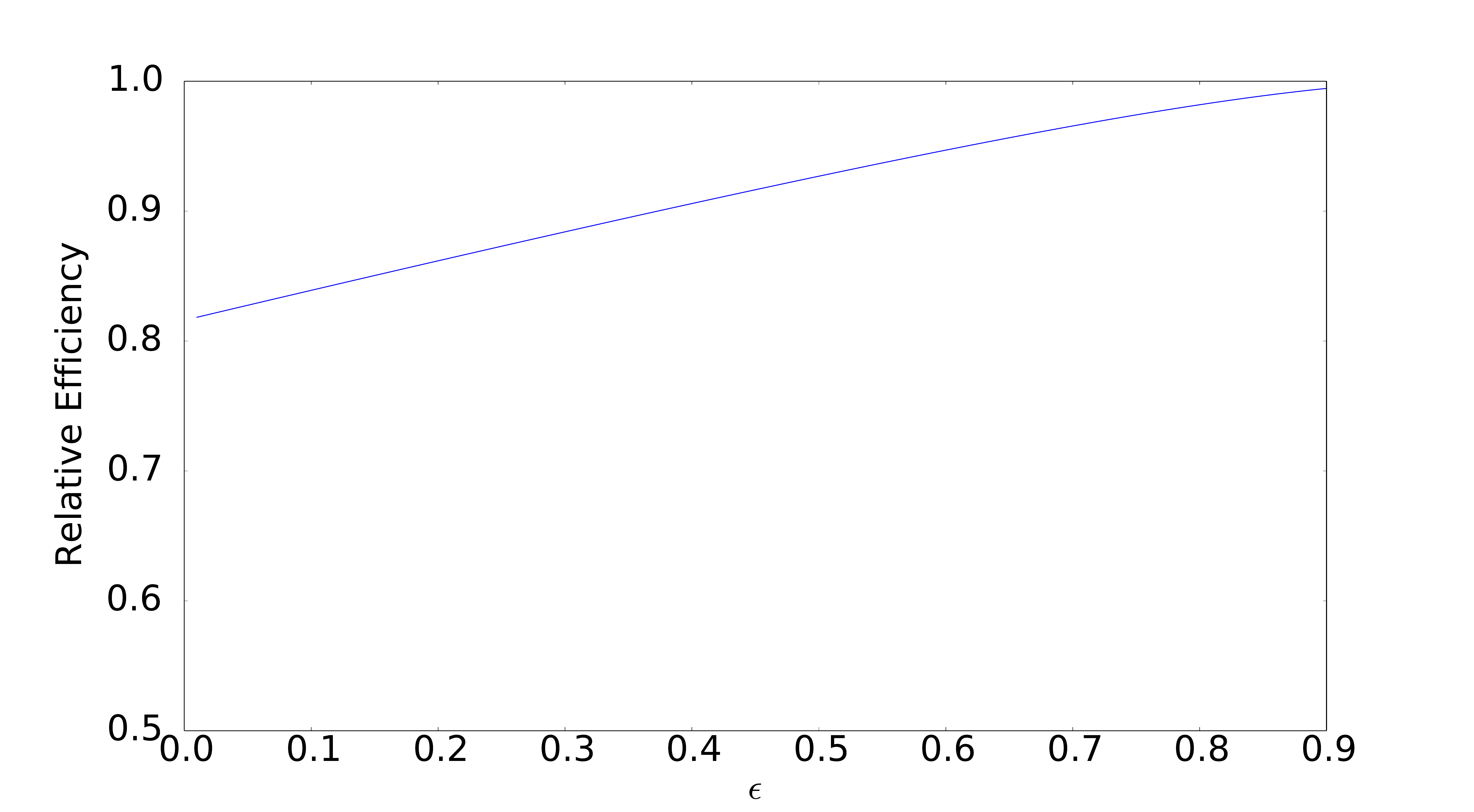}
      & \hspace{-.3cm}
      \includegraphics[width=.54\columnwidth]{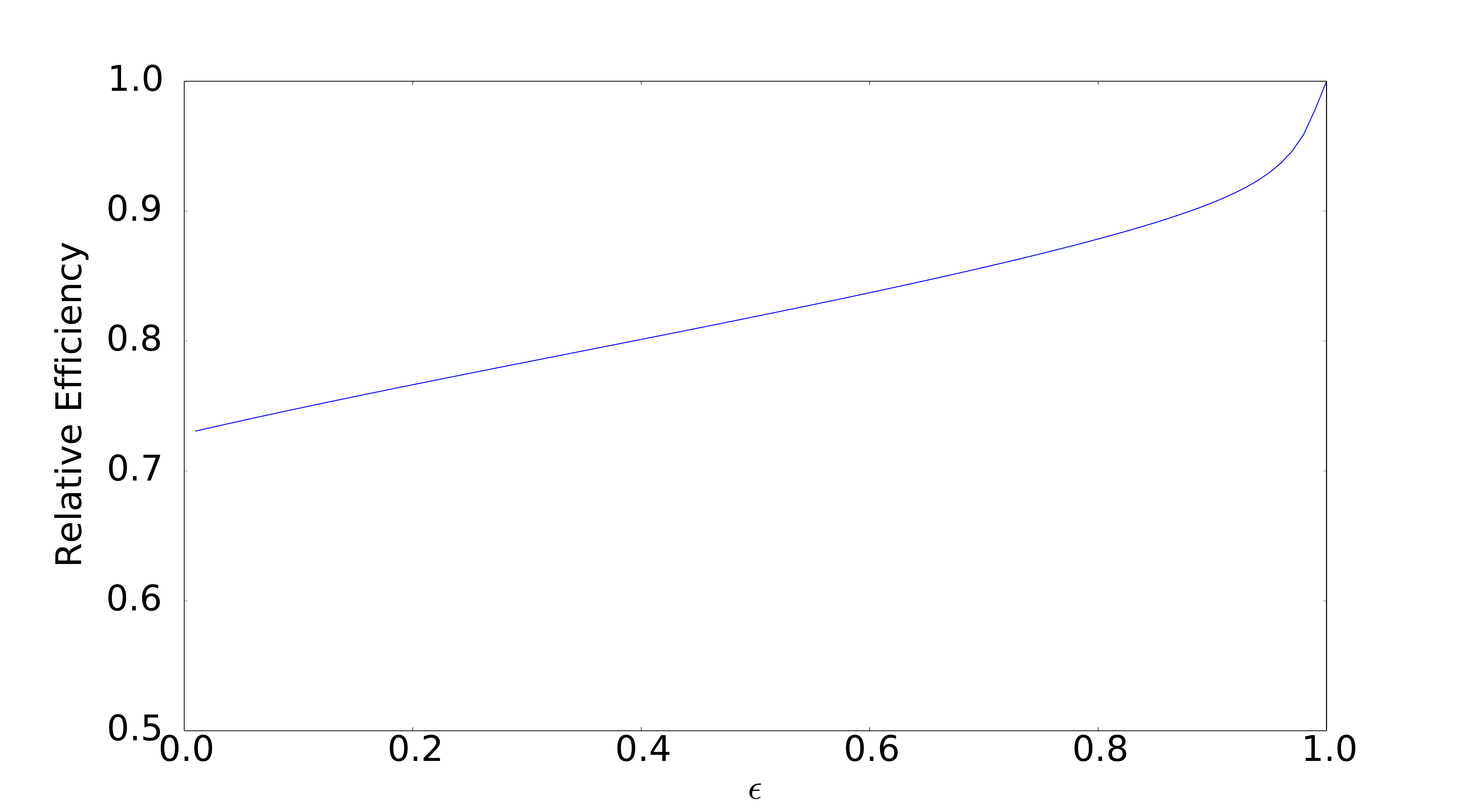} \\
      (a) & (b)
    \end{tabular}
    \caption{\label{fig:efficiency-rr}
      The efficiency of $\MomML$ relative to $\MargML$ as $\epsilon$ varies
      for weak (a) and strong (b) signals $\theta$.}
  \end{center}
\end{figure}

\paragraph{An empirical plot.}
The H\'ajek-Le Cam convolution and local asymptotic minimax theorems give
that $\MargML$ is the most statistically efficient estimator.  We now
empirically study the efficiency of $\MomML$ relative to $\MargML$, where
$\eff \eqdef d^{-1} \tr(\vMargML \vMomML^{-1})$, the average of the relative
variances per coordinate of $\MargML$ to $\MomML$.  We continue to focus on
randomized response in the unconditional case.

To study the effect of $\epsilon$, we consider the following probability
model: we let $y \in \{1, 2, 3, 4\}$, define
\begin{equation*}
  \phi(1) = \left[\begin{matrix} 0 \\ 0 \end{matrix} \right], ~~
  \phi(2) = \left[\begin{matrix} 0 \\ 1 \end{matrix} \right], ~~
  \phi(3) = \left[\begin{matrix} 1 \\ 0 \end{matrix} \right], ~~
  \phi(4) = \left[\begin{matrix} 1 \\ 1 \end{matrix} \right],
\end{equation*}
and set $p_\theta(y) \propto \exp(\theta^\top \phi(y))$.  We set $\theta =
[2, -0.1]^\top$ and $\theta = [5, -1]^\top$ to represent weak and strong
signals $\theta$ (the latter is harder to estimate, as the Fisher
information matrix is much smaller); when $\theta = 0$, the asymptotic
variances are equal, $\vMomML = \vMargML$. In
Figure~\ref{fig:efficiency-rr}, we see that the asymptotic efficiency of
$\MomML$ relative to $\MargML$ decreases as $\epsilon \to 0$, which is
explained by the fact that---as we see in
expression~\eqref{eqn:margVariance}---the $\MargML$ estimator leverages the
prior information about $y$ based on $\theta^*$, while as $\epsilon \to 0$,
expression~\eqref{eqn:variances-rr} is dominated by the $1 / \epsilon^2
\cov[\phi(y')]$ term, where $y'$ is uniform. Moreover, as $\theta$
grows larger, the conditional covariance $\cov[\phi(y) \mid o]$
is much smaller than the covariance $\cov[\beta(o) \mid y]$,
so that we expect that $\vMargML \preceq \vMomML$.

\section{The geometry of two-step estimation}
\label{sec:two_step}

\renewcommand{\comment}[1]{}

We now provide some geometric intuition
about the differences between $\MargML$ and $\MomML$,
establishing a connection between $\MomML$ and the EM algorithm
as a byproduct of our discussion.
For concreteness,
let $\sY = \{ 1, \dots, m \}$ be a finite set and let $\sP$ be the set of all
distributions over $\sY$ (represented as $m$-dimensional vectors).
Let $\sF \eqdef \{ p_\theta : \theta \in \R^d \} \subseteq \sP$ 
be a natural exponential family over $\mY$ with 
$p_\theta(y) \propto \exp(\theta^\top \phi(y))$.
See \reffig{surface_F} for an example where $m = 3$ and $d = 2$.
Note that in the space of distributions, $\sF$ is a non-convex set.

Let $\sO = \{ 1, \dots, k \}$ be the set of observations.
We can represent the supervision function $S(o \mid y)$ as a matrix
$S \in \R^{k \times m}$.
For $p \in \sP$, we can express the marginal distribution over $o$ as $q = S p$.
Let $\hat q = \inv{n} \sum_{i=1}^n \delta_{o^{(i)}}$ be the empirical
distribution over observations.

The maximum marginal likelihood estimator can now be written succinctly as:
\begin{align}
  \label{eqn:MargMLobj}
  \MargML = \argmin_{p \in \sF} \KL{\hat q}{S p}.
\end{align}
While the KL-divergence is concave in $p$, the non-convex constraint set $\sF$
makes the problem difficult.

\SpaceFig{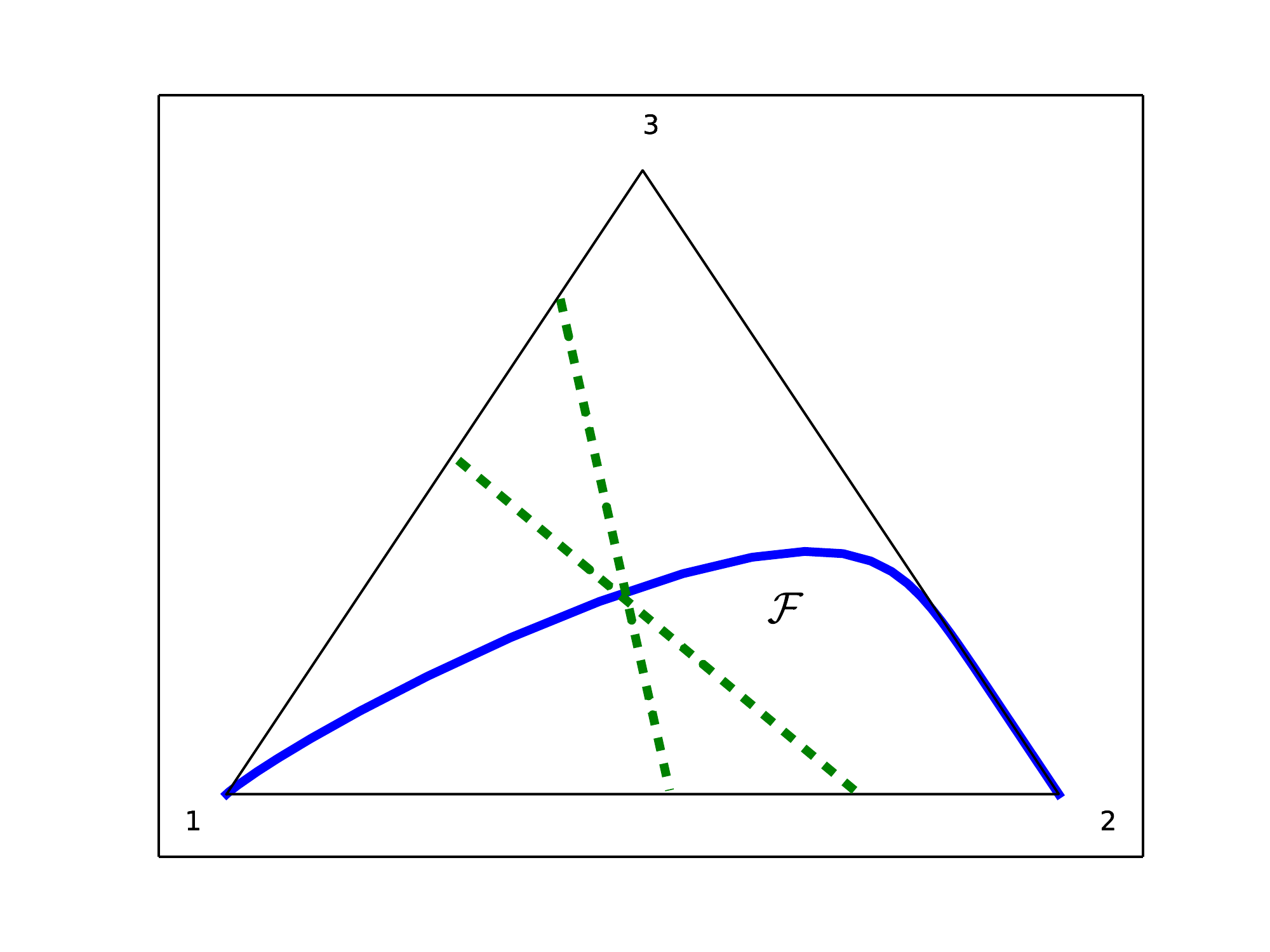}{0.25}{surface_F}{
  Visualization of the exponential
  family $\sF$ and all distributions $\sP$ over $\sY = \{ 1, 2, 3 \}$; here $\sP$ is the 3-simplex.
  The model features for $\sF$ are $\phi(1) = 5, \phi(2) = -1,
  \phi(3) = 0$. The blue curve marks out the exponential family $\sF =
  \{p_\theta : \theta \in \R \}$. Observations yield two moment
  equations (dotted green) whose intersection with the simplex pins down
  the data distribution.}

Our moment-based estimator $\MomML$ can be viewed as a relaxation, where we
first optimize over a relaxed set $\sP$ and then project onto the exponential
family:
\begin{align}
  \label{eqn:two-step-mom}
  \hat{r} = \argmin_{p \in \sP} D(\hat{q}, Sp),
  ~~
  \hat{p} = \argmin_{p \in \sF} \KL{\hat{r}}{p}.
\end{align}
The first step can be computed directly via $r = S^\dagger \hat q$ if $D$ is
the squared Euclidean distance.  If $D$ is KL-divergence, we can use EM
(see the composite likelihood objective of \citet{chaganty2014graphical}),
which converges to the global optimum.
The result is a single distribution $\hat r$ over $y$.
The second step optimizes over $p$ via $\theta$, which is a convex optimization problem,
resulting in $\hat p$ corresponding to $\MomML$.

\Fig{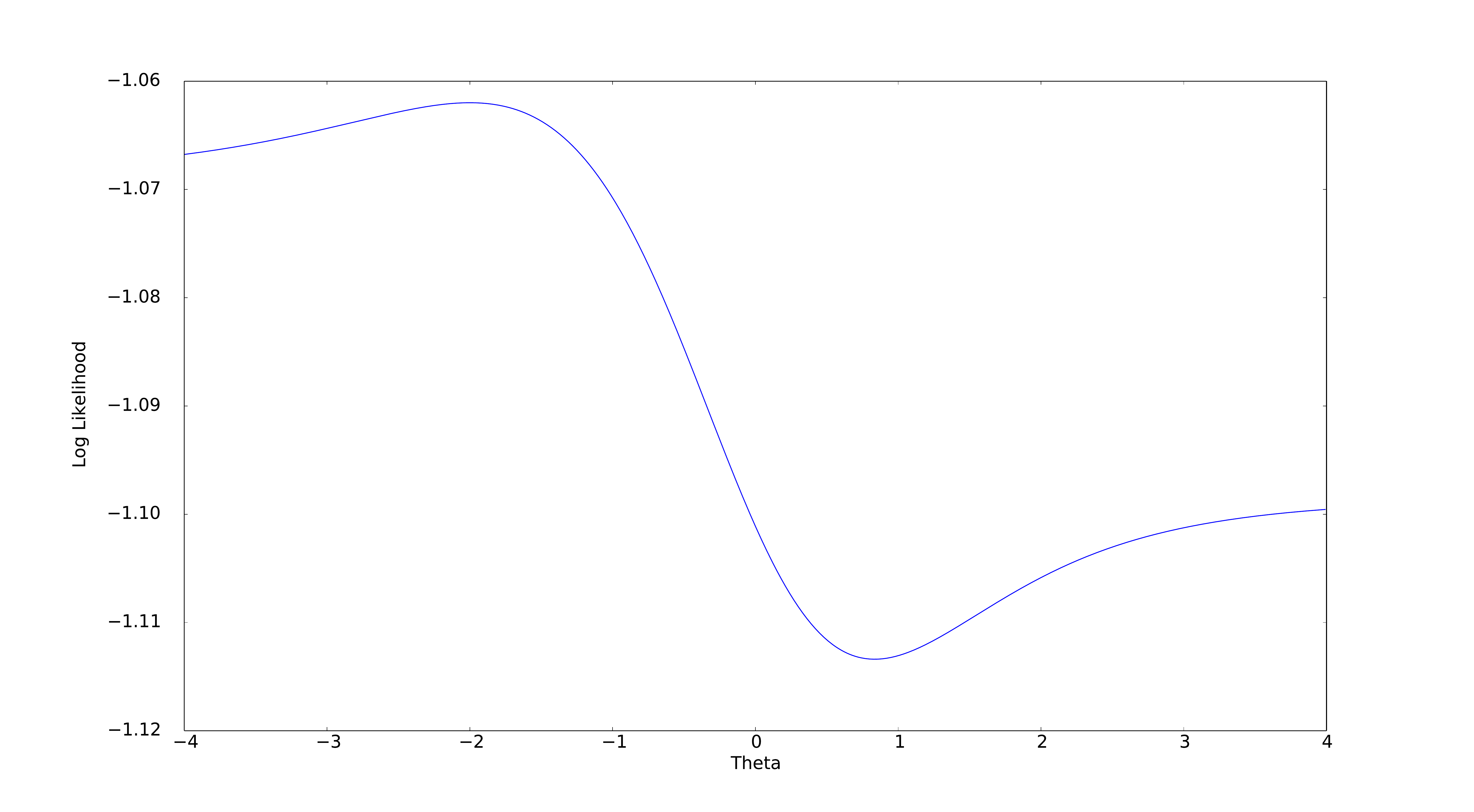}{0.08}{inv_LL}
{Log-marginal likelihood $\log \sum_{y=1}^3 S(o \mid y) p_\theta(y)$,
where the exponential family features are $\phi(1) = 2$, $\phi(2) = 1$, $\phi(3) = 0$. 
The model is well specified with
$S =[
  \tfrac 1 3 ~ \tfrac 1 6 ~ \tfrac 1 4 ~;~
  \tfrac 1 3 ~ \tfrac 1 6 ~ \tfrac 1 2 ~;~
  \tfrac 1 3 ~ \tfrac 2 3 ~ \tfrac 1 4 ]$.
}

Computing $\MargML$ generally requires solving a non-convex optimization
problem (see \reffig{inv_LL} for an example).  When $S$ has full column rank
and the model is well-specified, $\MomML$ is consistent: we have
that $\hat r = \hat{p} = S^\dagger \hat{q} \cvP S^\dagger S p^*
= p^*$. This means that eventually the KL projection of
problem~\eqref{eqn:two-step-mom} is essentially an identity operation: we
almost have $\hat{r} \in \sF$ by the rank assumptions, making the problem
easy.  This assumption strongly depends on the well-specifiedness of the
supervision; indeed, if $q^* \neq S p$ for any $p \in \sP$, then $\|\MargML
- \MomML\| \ge c > 0$, for a constant $c$, even as $n \to \infty$.  We can
relax the column rank assumption, however: $S$ simply needs to
contain enough information about the sufficient statistics, that is, if
$\Phi = [\phi(1) \cdots \phi(m)] \in \R^{d \times m}$ is matrix of sufficient
statistics, we require that $\Phi = S R$ for some matrix $R$.

\paragraph{Deterministic supervision.}

When the supervision matrix $S$ has full column rank,
$\MomML$ converges to $\theta^*$.  There are certainly cases
where $\MargML$ is consistent, but $\MomML$ is not.
What can we say about $\MomML$ in this case?

To obtain intuition, consider the case
the supervision is a deterministic function that maps $y$
to $o$ (region annotations is an example).
In this case, every column of $S$ is an indicator vector,
and $S^\dagger =  S^\top \diag(S \mathbf 1)^{-1} =
S^\top (SS^\top)^{-1}$.

Here, $S^\dagger$ 
distributes probability mass evenly across all the 
$y \in \mY$ that deterministic map to $o$.
In this case, $\MomML$ simply corresponds
to running one iteration of EM on the marginal likelihood,
initializing with the uniform distribution over $y$ ($\theta = 0$).
The E-step conditions on $o$ and places uniform mass over consistent $y$, producing $\hat r$;
the M-step optimizes $\theta$ based on $\hat r$.

\section{Experiments}
\label{sec:experiments}

\SpaceFigTop{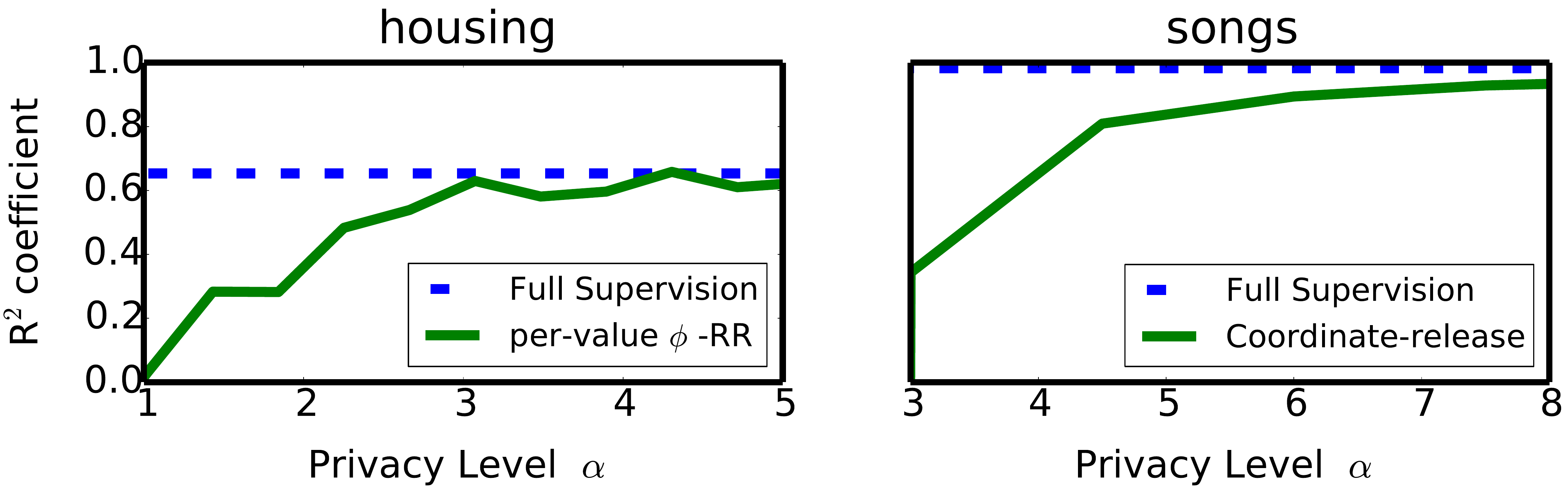}{0.19}{privacy_plot}{ R$^2$ coefficient
  for linear regression when estimating from privately revealed
  sufficient statistics on two datasets.}

\paragraph{Local privacy.}
Following Section~\ref{sec:localPrivacy}, we consider locally private
estimation of a structured model.
We take linear
regression as a simple such structured model: it corresponds to a
pairwise random field over the inputs and the response. The sufficient
statistics are edge features $\phi_{i,j}(x_i,x_j) = x_i x_j$ and $\phi_{i}(x_i,y)
= x_i y$ for each $i,j \in [d]$. 

On the housing dataset, the supervision $o$ is given under the
per-value RR scheme. On the songs dataset, $\tilde{o}$ is a dense
vector, motivating the coordinate release scheme instead.  We choose
$i \in [d]$ at random, with probability $\half$ reveal $x_i y$ and
with probability $\half$ reveal $[x_i^2, x_i x_j, x_j^2]$ with
suitable noise as described in Section~\ref{sec:structuredRR}.  Note
that the noising mechanism privatizes both input variables $x$ as well
as the response $y$.

Figure~\ref{fig:privacy_plot} visualizes the average (over 10 trials)
R$^2$ coefficient of fit for linear regression on the test
set,\footnote{The (uncentered) R$^2$ coefficient of parameters $w$ in a linear
  regression with design $X$ and labels $Y$ is
  $\|Xw-Y\|^2/\|Y\|^2$.}  in response to varying the privacy
parameter $\alpha$.\footnote{We use the housing (\url{mlcomp.org/datasets/840}) and
  songs (\url{mlcomp.org/datasets/748}) data from~\href{http://mlcomp.org}{mlcomp}.}
As expected, the efficiency degrades with increase in the privacy
constraint, though for moderate values of $\alpha$ the loss is not
significant.

\paragraph{Lightweight annotations.}
We experiment with estimating a conditional model for part-of-speech
(POS) sequence tagging from lightweight
annotations.\footnote{We used the Wall Street Journal portion of the
  Penn Treebank. Sections 0-21 comprise the training set and 22-24
  are test.}
Every example in the dataset reveals a sentence and the counts of all tags over a
consecutive window. Following the modeling assumptions in
Section~\ref{sec:annotations}, we use a CRF (per
Section~\ref{sec:setup}) with only node features:
\begin{align*}
  \phi_j(y[j], x, j)[g, a, b] &= \sum_{i \in [L], g \in \mG} \1[g(x_i) = a, y[i] = b],
\end{align*}
where $g$ is a function on the word (e.g., word, prefix, suffix and
word signature). 

When the problem is fully supervised, we maximize the log-likelihood
with stochastic gradient descent (SGD); in this case, estimation
is convex and exact gradients can be tractably computed.
Under count supervision, convexity of the marginal
likelihood is not guaranteed. Although the model
has no edge features, the indirect count supervision places an
potential over the region in which counts are revealed (one
enforcing that the tag sequence is compatible with the counts).
This renders exact inference intractable,
so we approximate it using beam search to compute
stochastic gradients.\footnote{The dataset has 45 tag values. We use a
  beam of size 500 after analytically marginalizing nodes outside the region.}
The moment-based estimator is unaffected by this issue as it
requires no inference and proceeds via a pair convex minimization
programs; we minimize both using SGD.

\SpaceFigTop{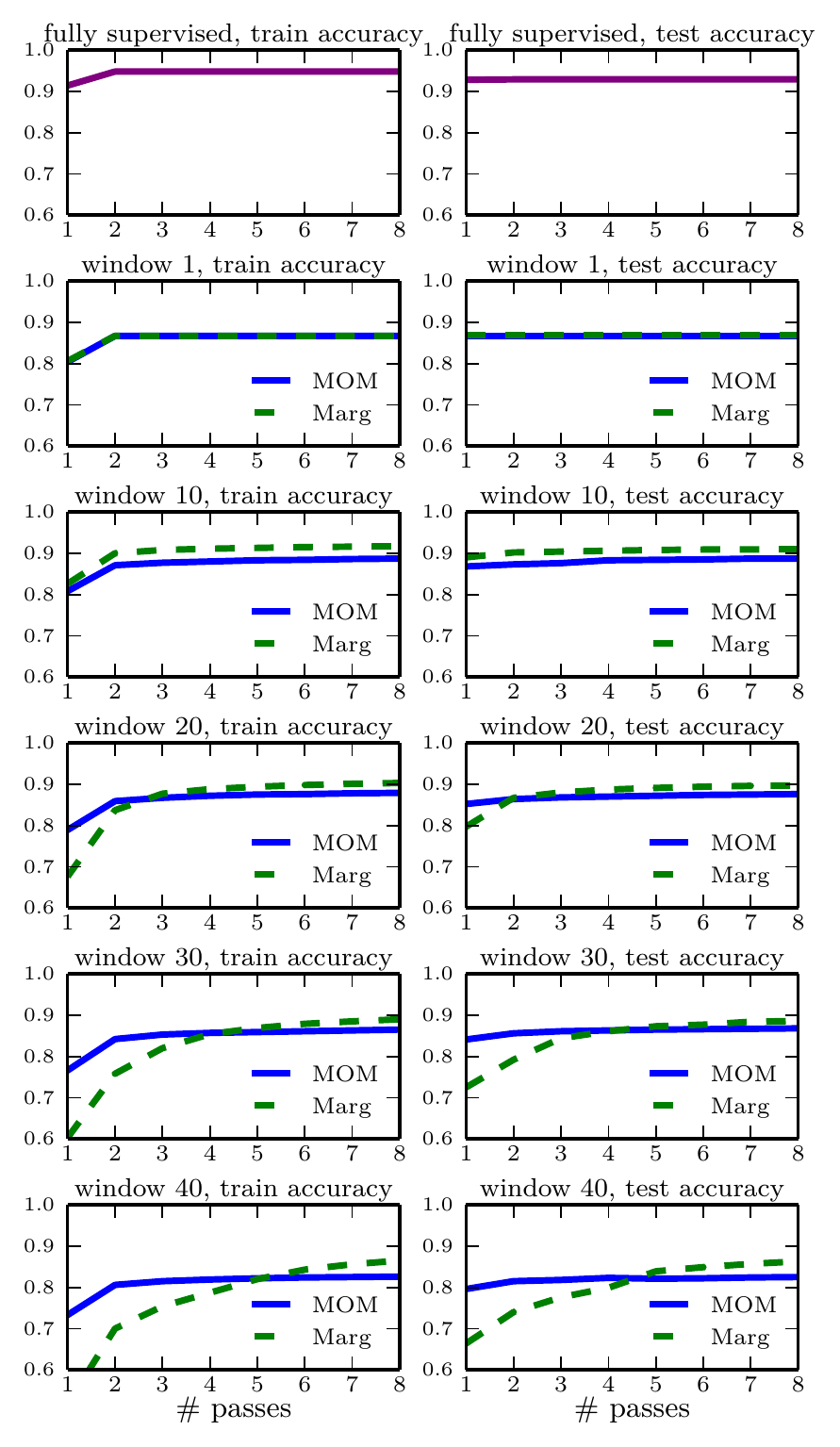}{0.82}{pos_plot}{
  Train and test per-position accuracies for $\MargML$ and $\MomML$
  on part-of-speech tagging,
  under various sized regions of count annotations,
  as training passes are taken through the dataset.}

Figure~\ref{fig:pos_plot} shows train and test accuracies as
we make passes over the dataset.
Typically, after sufficiently many passes, the
marginal likelihood gains an advantage over the moment-based
estimator.
For small regions, we expect the beam search approximation to
be accurate, and indeed the marginal likelihood estimator is
dominant there. For larger regions,
the moment-based estimator (i) achieves high accuracy early
and (ii) dominates for several passes before the marginal likelihood
estimator overtakes it. Altogether, the experiment highlights that the
moment-based estimator is favorable in computationally-constrained settings.

\section{Related work and discussion}
\label{sec:discussion}

This work was motivated by two use cases of indirect supervision: local privacy and cheap annotations.
Each trades off statistical accuracy for another resource: privacy or annotation cost.
Local privacy has its roots in classical randomized response for conducting surveys \citep{warner1965randomized},
which has been extended to the multivariate \citep{tamhane1981randomized} and
conditional \citep{matloff1984use} settings.  In the computer science community,
differential privacy has emerged as a useful formalization of privacy \citep{dwork2006differential}.
We work with the stronger notion of local differential privacy
\citep{evfimievski2004privacy,kasiviswanathan2011can,duchi2013local}.
Our contribution here is two-fold: First, we bring local privacy to the graphical
model setting, which provides an opportunity for the privacy mechanism to be sensitive
to the model structure.  While we believe our mechanisms are reasonable,
an open question is designing \emph{optimal} mechanisms in the structured case.
Second, we connect privacy with other forms of indirect supervision.

The second use case is learning from lightweight annotations,
which has taken many forms in the literature.
Multi-instance learning \citep{maron1998framework} is popular in computer vision,
where it is natural to label the presence but not location of objects \citep{babenko2009visual}.
In natural language processing, there also been work on 
learning from structured outputs where, like this work, only counts of labels
are observed \citep{mann08ge,liang09measurements}.
However, these works resort to likelihood-based approaches which involve non-convex optimization
and approximate inference, whereas in this work,
we show that linear algebra and convex optimization suffice under modeling assumptions.

\citet{quadrianto08labels} showed how to learn from label proportions of groups of examples,
using a linear system technique similar to ours.  However, they assume
that the group is conditionally independent of the example given the label,
which would not apply in our region-based annotation setup since our regions contain arbitrarily
correlated inputs and heterogeneous labels.  In return, we do need to make the stronger assumption
that each label $y[i]$ depends only on a discrete $x[i]$,
so that the credit assignment can be done using a linear program.
An open challenge is to allow for heterogeneity with complex inputs.

Indirect supervision arises more generally in latent-variable models,
which arises in machine translation \citep{brown93mt},
semantic parsing \citep{liang11dcs},
object detection \citep{quattoni04crf},
and other missing data problems in statistics \citep{robins2000inference}.
The indirect supervision problems in this paper have additional structure:
we have an unknown model $p_\theta$ and a \emph{known} supervision function $S$.
It is this structure allows us to obtain computationally efficient method of moments procedures.

We started this work to see how much juice we could squeeze out of just linear moment equations,
and the answer is more than we expected.
Of course, for more general latent-variable models beyond linearly indirectly-supervised problems,
we would need more powerful tools.
In recent years, tensor factorization techniques have 
provided efficient methods
for a wide class of latent-variable models \citep{hsu12identifiability,anandkumar12lda,hsu13spherical,anandkumar13tensor,chaganty13regression,halpern2013unsupervised,chaganty2014graphical}.
One can leverage even more general polynomial-solving techniques to expand the set of models
\citep{wang2015polynomial}.
In general, the method of moments allows us to leverage statistical structure to alleviate
computational intractability, and we anticipate more future developments along these lines.

\paragraph{Reproducibility.}The code, data and experiments for this paper are available
on Codalab at \url{https://worksheets.codalab.org/worksheets/0x6a264a96efea41158847eef9ec2f76bc/}.

\bibliography{all}
\bibliographystyle{icml2016}

\clearpage
\onecolumn
\appendix

\section{Details of privacy schemes}
\label{sec:proofs}

\subsection{Local privacy using sufficient statistics}
\label{sec:privacy_proofs}

\begin{proof}[Proof of Proposition~\ref{prop:sufficient_channel}]
  Because $\phi$ is a sufficient statistic, by definition
  there exists some channel $Q(y \mid \phi(x, y))$ and
  a distribution $F_\theta(\phi(x, y) \mid x)$ such that
  $p_\theta(y \mid x) = Q(y \mid \phi(x, y)) F_\theta(\phi(x, y) \mid x)$.
  If we define
  \begin{align}
    S'(o \mid \phi(x, y)) = \sum_y S(o \mid y) Q(y \mid \phi(x, y)),
  \end{align}
  then \refeqn{preserveXO} follows by substitution and algebra:
\end{proof}

\subsection{Privacy guarantees of proposed schemes}

In order to show differential privacy of the two schemes proposed in
Section~\ref{sec:localPrivacy}, we first note that it suffices to have
differential privacy of the observations $o$ with respect to any
(possibly random) data $\tilde{o} \in \tilde{\mO}$ processed given the private
variable $y$ such that $y \to \tilde{o} \to o$ forms a Markov chain.

To see this, suppose $Q$ is an $\diffp$-differentially private channel
taking the intermediate variable $\tilde{o}$ to $o$ and fix any $x \in
\mX$. Let $R(\cdot \mid y)$ be the distribution of $\tilde{o}$ given $y \in
\mY$. Now, for the end-to-end channel $S$,
\begin{align}
  \sup_{o, y, y'} \frac{
    S( o \mid y )
  }{
    S( o \mid y' )
  }
  &=
  \sup_{o, y, y'} \frac{
    \sum_{\tilde{o} \in \tilde{\mO}} Q(o \mid \tilde{o}) R(\tilde{o} \mid y)
  }{
    \sum_{\tilde{o} \in \tilde{\mO}} Q(o \mid \tilde{o}) R(\tilde{o} \mid y')
  } \\
  &\le
  \sup_{o, y, y'} \frac{
    \max_{\tilde{o}} Q(o \mid \tilde{o})
  }{
    \min_{\tilde{o}} Q(o \mid \tilde{o})
  } \\
  &\le \exp(\diffp).
\end{align}

\paragraph{Coordinate release.}
Recall that in the coordinate release mechanism, we first pick a 
coordinate $j$ and release observation $o_\coordrel$ after flipping
$\tilde{o}[j]$ with probability $\frac {\exp ( \frac{\alpha}{2}) } {1 + \exp (\frac {\alpha } {2 }  )  } $.

\begin{align}
  \frac{
    Q(o_\coordrel, j \mid \tilde o)
  }{
    Q(o_\coordrel, j \mid \tilde o')
  }
  &=\exp\left( \frac \diffp 2 \left( |o_\coordrel - (1-\tilde o[j])| -  |o_\coordrel - (1-\tilde o'[j])| \right) \right)\\
  &\le
  \exp(\alpha),
\end{align}
where the final step is by the triangle inequality applied twice.

\paragraph{Per-value $\phi$-RR.}Privacy of per-value $\phi$-RR follows similarly.

Each coordinate of $o$ is flipped with probability $q_{ \frac { \alpha } {\bar \delta } } =  \frac {\exp ( \frac{\alpha}{2 \bar \diameter}) } {1 + \exp (\frac {\alpha } {2 \bar \diameter }  )  } $, where
$\bar \diameter$ is chosen such that $\bar \diameter \leq \lone{\tilde {o}}, \lone{\tilde{o'}}$ (see Section ~\ref{sec:structuredRR})
\begin{align}
  \frac{
    Q(o_\pervalue \mid \tilde o)
  }{
    Q(o_\pervalue \mid \tilde o')
  }
  &=
  \exp\left( \frac \diffp {2\bar\diameter} \left( \lone{o_\pervalue - (\mathbf{1} - \tilde o)} - \lone{o_\pervalue - (\mathbf{1} - \tilde o')} \right) \right) \\
  &\le
  \exp(\alpha).
\end{align}

\subsection{Variance of moments-based estimator for different privacy schemes.}
\label{sec:variance_appendix}
For simplicity, we once again consider the unconditional case (where $x$ is empty) and assume $\phi \in \{0, 1\}^d$.

\begin{theorem}[Asymptotic variance (coordinate release)]
The asymptotic variance of $\MomML$ for
$\alpha$-differentially private coordinate release scheme, under a
uniform coordinate sampling distribution $p_\coordrel$ is
\begin{align*}
  \vMomML^\coordrel &= \sI^{-1} + \sI^{-1} H^\coordrel \sI^{-1},
\end{align*}
where
\begin{align}
  \label{eqn:variance_coordrel}
  H^\coordrel &= \frac{d q_\alpha (1 - q_\alpha) } { (2 q_\alpha - 1)^2} I  + \E[ d \diag(\phi(y))^2 - \phi(y)\vecsquare],
\end{align}
\end{theorem}

As in Lemma~\ref{lem:variancerr}, the matrix $H^\coordrel$ governs the
loss in efficiency under the coordinate release mechanism, which
arises from two sources: (i) variance due to the stochastic flipping
process and (ii) variance due to choosing a random coordinate for
release.

\begin{proof}
When $p_\coordrel$ is uniform, the observation function $\beta_{\coordrel}{(j, o_\coordrel)}$ takes the following form. 
\begin{align*}
\beta_{\coordrel}(j, o_\coordrel) =
d \;
\frac { o_\coordrel - 1 + q_{\diffp} } { 2 q_{\diffp} - 1 }  \;
e_j.
\end{align*}
From \refeqn{momVariance}, we have that $H^\coordrel = \E [ \cov[ \beta_\coordrel(j, o_\coordrel) \mid y] ]$.

We decompose $\cov[\beta(j, o_\coordrel) \mid y]$ as 
\begin{align*}
  &\E [\cov[\beta(j, o_\coordrel) \mid j, y] \mid y] + \cov [ \E[ \beta(j, o_\coordrel) \mid j, y ] \mid y] \\
  & = \frac{d^2} { (2 q_\alpha - 1)^2 } \big( \E [\cov[o_\coordrel ~e_j \mid j, y] \mid y] + \cov [ \E[  (o_\coordrel - 1 + q_\alpha) ~e_j \mid j, y ] \mid y] \big) \\
  & = \frac{d^2} { (2 q_\alpha - 1)^2 }  \E [ \diag ( q_\alpha (1 - q_\alpha) e_j) \mid y]  + \frac{d^2} { (2 q_\alpha - 1)^2 } \cov [ [(2 q_\alpha - 1)\phi(y)[j] + 1 - q_\alpha - 1 + q_\alpha]e_j \mid y]  \\
  & = \frac{d^2} { (2 q_\alpha - 1)^2 }  \E [ \diag ( q_\alpha (1 - q_\alpha) e_j) \mid y]  + d^2 \cov [ \phi(y)[j] e_j \mid y] \\
  &= \frac{d q_\alpha (1 - q_\alpha) } { (2 q_\alpha - 1)^2} I  +  [ d \diag(\phi(y))^2 - \phi(y)\vecsquare], 
\end{align*}
\end{proof}

\begin{theorem}[Asymptotic variance (per-value $\phi$-RR)]
The asymptotic variance of $\MomML$ for $\alpha$-differentially per-value $\phi$- RR scheme is
\begin{align}
  \vMomML^\pervalue &= \sI^{-1} + \sI^{-1} H^\pervalue \sI^{-1},  \\
  H^\pervalue &= \frac{ q_{\diffp/\bar\diameter} ( 1 -  q_{\diffp/\bar\diameter})} { (2  q_{\diffp/\bar\diameter} - 1)^2} I.   \label{eqn:variance_pervalue}
\end{align}
\end{theorem}

\begin{proof}
From \refeqn{pervalue-beta-construction}, we have
\begin{align*}
  \beta_{\pervalue}(x, o_\pervalue) \eqdef
  \frac { o_\pervalue - \mathbf{1} + q_{\diffp/\bar\diameter} \mathbf{1} } { 2 q_{\diffp/\bar\diameter} - 1 } \;. 
\end{align*} 

From \refeqn{momVariance}, we know that $H^\pervalue = \E [ \cov[ \beta_\pervalue(o_\pervalue) \mid y] ] = \frac {1}{(2q_{\diffp/\bar\diameter} - 1)^2 } \E [ \cov[ o_\pervalue \mid y] ]$. 

Each entry $o_\pervalue[j]$ is chosen independently according to 
a Bernoulli distribution with parameter $q_{\diffp/\bar\diameter}$ (if $\tilde{o}[i] = 1$) or $1 - q_{\diffp/\bar\diameter}$ (if $\tilde{o}[i] = 0$), 
implying the claim.
\end{proof}

\subsection{Comparison of the two schemes}
\label{sec:comparison}
We use $\tr(H^\coordrel)$ and $\tr(H^\pervalue)$ to quantitatively estimate the loss 
in efficiency of $\vMomML$ under the two privacy schemes. 

For small $x$, approximate $q_x = \frac {e^{x/2}} {1 + e^{x/2}}$ locally as $\frac{1}{2} + \frac{x}{8}$ (via Taylor expansion).
Substituting in~\refeqns{variance_coordrel}{variance_pervalue} gives the following
expressions for small $\alpha$:
\begin{align} 
\tr(H^\coordrel) &\approx \frac{4 d^2}{ \alpha^2} + \bar\diameter (d - 1) . \\
\tr(H^\pervalue) &\approx \frac{4 d \bar\diameter^2}{\alpha^2}.
\end{align}

When $\bar\diameter$ is constant (independent of d),
$\tr(H^\pervalue)$ grows linearly with $d$ whereas $\tr(H^\coordrel)$
grows quadratically with $d$. Therefore per-value $\phi$-RR has
smaller loss when $\phi$ has low $l_1$ norm. Meanwhile, when
$\bar\diameter = O(d)$, $\tr(H^\pervalue) = O(d^3)$ and
$\tr(H^\coordrel) = O(d^2)$. Hence coordinate release is a more
appealing choice if $\phi$ is dense.

\end{document}